\documentclass[10pt,twocolumn,letterpaper]{article}

\usepackage{cvpr}
\cvprfinalcopy
\usepackage{amsmath}
\usepackage{amssymb}
\usepackage{helvet}
\usepackage{courier}

\usepackage[utf8]{inputenc} 
\usepackage[T1]{fontenc}    
\usepackage{booktabs}   
\usepackage{amsfonts}       
\usepackage{nicefrac}       
\usepackage{microtype}      
\usepackage[title]{appendix}

\usepackage{times}
\usepackage{xcolor}
\usepackage{bm}
\usepackage{soul}
\usepackage{multirow}
\usepackage{adjustbox}
\usepackage{xspace}
\usepackage{enumitem}
\usepackage{amsthm}
\usepackage{subcaption}

\newcommand{\rr}{\mathbb{R}}
\newcommand{\E}{\bm{\epsilon}}

\newcommand{\A}[1]{\mathbf{A}^{#1}}

\newcommand{\Att}{\mathcal{A}}
\newcommand{\dA}{d_\mathcal{A}}

\newtheorem{theorem}{Theorem}[section]

\newtheorem{remark}[theorem]{Remark}
\newtheorem{claim}[theorem]{Claim}

\usepackage[pagebackref=true,breaklinks=true,colorlinks,bookmarks=false]{hyperref}

\def\cvprPaperID{9633} 

\usepackage{authblk}
\makeatletter
\renewcommand\AB@affilsepx{, \protect\Affilfont}
\makeatother

\title{Towards Verifying Robustness of Neural Networks Against \\A Family of Semantic Perturbations}

\author[1]{Jeet Mohapatra}
\author[1]{Tsui-Wei Weng}
\author[2]{Pin-Yu Chen}
\author[2]{Sijia Liu}
\author[1]{Luca Daniel}
\affil[1]{MIT EECS} 
\affil[2]{MIT-IBM Watson AI Lab, IBM Research}

\begin{document}

\maketitle

\frenchspacing

\begin{abstract}
Verifying robustness of neural networks given a specified threat model is a fundamental yet challenging task. While current verification methods mainly focus on the $\ell_p$-norm threat model of the input instances, robustness verification against semantic adversarial attacks inducing large $\ell_p$-norm perturbations, such as color shifting and lighting adjustment, are beyond their capacity. To bridge this gap, we propose \textit{Semantify-NN}, a model-agnostic and generic robustness verification approach against semantic perturbations for neural networks. By simply inserting our proposed \textit{semantic perturbation layers} (SP-layers) to the input layer of any given model, \textit{Semantify-NN} is model-agnostic, and any $\ell_p$-norm based verification tools can be used to verify the model robustness against semantic perturbations. We illustrate the principles of designing the SP-layers and provide examples including semantic perturbations to image classification in the space of hue, saturation, lightness, brightness, contrast and rotation, respectively. In addition, an efficient refinement technique is proposed to further significantly improve the semantic certificate. Experiments on various network architectures and different datasets demonstrate the superior verification performance of \textit{Semantify-NN} over $\ell_p$-norm-based verification frameworks that naively convert semantic perturbation to $\ell_p$-norm. 
The results show that \textit{Semantify-NN} can support robustness verification against a wide range of semantic perturbations. Code : \href{https://github.com/JeetMo/Semantify-NN}{\color{blue}{https://github.com/JeetMo/Semantify-NN}}
\end{abstract}

\section{Introduction}
As deep neural networks (DNNs) become prevalent in machine learning and achieve the best performance in many standard benchmarks, their unexpected vulnerability to adversarial examples has spawned a wide spectrum of research disciplines in adversarial robustness, spanning from effective and efficient methods to find adversarial examples for causing model misbehavior (i.e., attacks), to detect adversarial inputs and become attack-resistant (i.e., defenses), and to formally evaluate and quantify the level of vulnerability of well-trained models (i.e., robustness verification~\cite{katz2017reluplex}, certification~\cite{weng2018towards}, or evaluation~\cite{weng2018evaluating}).

Given a data sample $x$ and a trained DNN, the primary goal of verification tools is to provide a ``robustness certificate'' for verifying its properties in a specified threat model. For image classification tasks, the commonly used threat model is an $\ell_p$-norm bounded perturbation to $x$, where $p$ usually takes the value $p\in \{1,2,\infty\}$, approximating the similarity measure of visual perception between $x$ and its perturbed version. The robustness property to be verified is the consistent decision making of the DNN of any sample drawn from an $\ell_p$-norm ball centered at $x$ with radius $\epsilon$. In other words, verification methods aim to verify whether the DNN can give the same top-1 class prediction to all samples in the $\epsilon$-ball centered at $x$. Note that verification is attack-agnostic as it does not incur any attack methods for verification. Moreover, if the $\epsilon$-ball robustness certificate for $x$ is verified, it assures no adversarial attacks using the same threat model can alter the top-1 prediction of the DNN for $x$. Although finding the maximal verifiable $\epsilon$ value (i.e. the minimum distortion) is computationally intractable for DNNs \cite{katz2017reluplex}, recent verification methods have developed efficient means for computing an lower bound on minimum distortion as a verifiable $\epsilon$-ball certificate \cite{kolter2017provable,wong2018scaling,weng2018towards,dvijotham2018dual,zhang2018crown,singh2018fast,wang2018efficient,raghunathan2018semidefinite,Boopathy2019cnncert}. 

Beyond the $\ell_p$-norm bounded threat model, recent works have shown the possibility of generating \textit{semantic} adversarial examples based on semantic perturbation techniques such as color shifting, lighting adjustment and rotation \cite{hosseini2018semantic,liu2018beyond,bhattad2019big,joshi2019semantic,fawzi2015manitest,engstrom2017rotation}. We refer the readers to Figure \ref{fig:veri_semantic} for the illustration of some semantic perturbations for images.
Notably, although semantically similar, these semantic adversarial attacks essentially consider different threat models than $\ell_p$-norm bounded attacks in the RGB space. Therefore, semantic adversarial examples usually incur large $\ell_p$-norm perturbations to the original data sample and thus exceed the verification capacity of $\ell_p$-norm based verification methods. To bridge this gap and with an endeavor to render robustness verification methods more inclusive, we propose \textit{Semantify-NN}, a model-agnostic and generic robustness verification against semantic perturbations. \textit{Semantify-NN} is model-agnostic because it can apply to any given trained model by simply inserting our designed \textit{semantic perturbation layers} (SP-layers). It is also generic since after adding SP-layers, one can apply any $\ell_p$-norm based verification tools for certifying semantic perturbations. In other words, our proposed SP-layers work as a carefully designed converter that transforms semantic threat models to $\ell_p$-norm threat models. As will be evident in the experiments, \textit{Semantify-NN} yields substantial improvement over $\ell_p$-norm based verification methods that directly convert semantic perturbations to the equivalent $\ell_p$ norm perturbations in the RGB space.

Our main contributions   are  summarized as below:
\begin{itemize}[leftmargin=*]
    \item We propose \textit{Semantify-NN}, a model-agnostic and generic robustness verification toolkit for semantic perturbations. \textit{Semantify-NN} can be viewed as a powerful extension module consisting of novel \textit{semantic perturbation layers} (SP-layers) and is compatible to existing $\ell_p$-norm based verification tools. The results show that \textit{Semantify-NN} can support robustness verification against a wide range of semantic perturbations. 
    
    \item We elucidate the design principles of our proposed SP-layers for a variety of semantic attacks, including  hue/saturation/lightness change in color space, brightness and contrast adjustment, rotation, translation and occlusion. We also propose to use input space refinement and splitting methods to further improve the performance of robustness verification. In addition, we illustrate the need and importance of robustness verification for continuously parameterized perturbations.
    
    \item We propose an efficient refinement technique, \textit{input splitting}, that can further tighten the semantic certificate delivered by \textit{Semantify-NN}. Our extensive experiments evaluated on the rich combinations of three datasets (MNIST, CIFAR-10 and GTSRB) and five different network architectures (MLPs and CNNs) corroborate the superior verification performance of \textit{Semantify-NN} over naive $\ell_p$-norm based verification methods. In particular, our method without further refinement can already achieve around 2-3 orders of magnitude larger (tighter) semantic robustness certificate than the baselines that directly uses the same $\ell_p$-norm verification methods to handle semantic perturbations. With our \textit{input splitting} technique, the semantic robustness certificate can be further improved by 100-300\% .   

\end{itemize}

\vspace{-2mm}


\begin{figure}[t]
    \centering
    \includegraphics[width=1\linewidth]{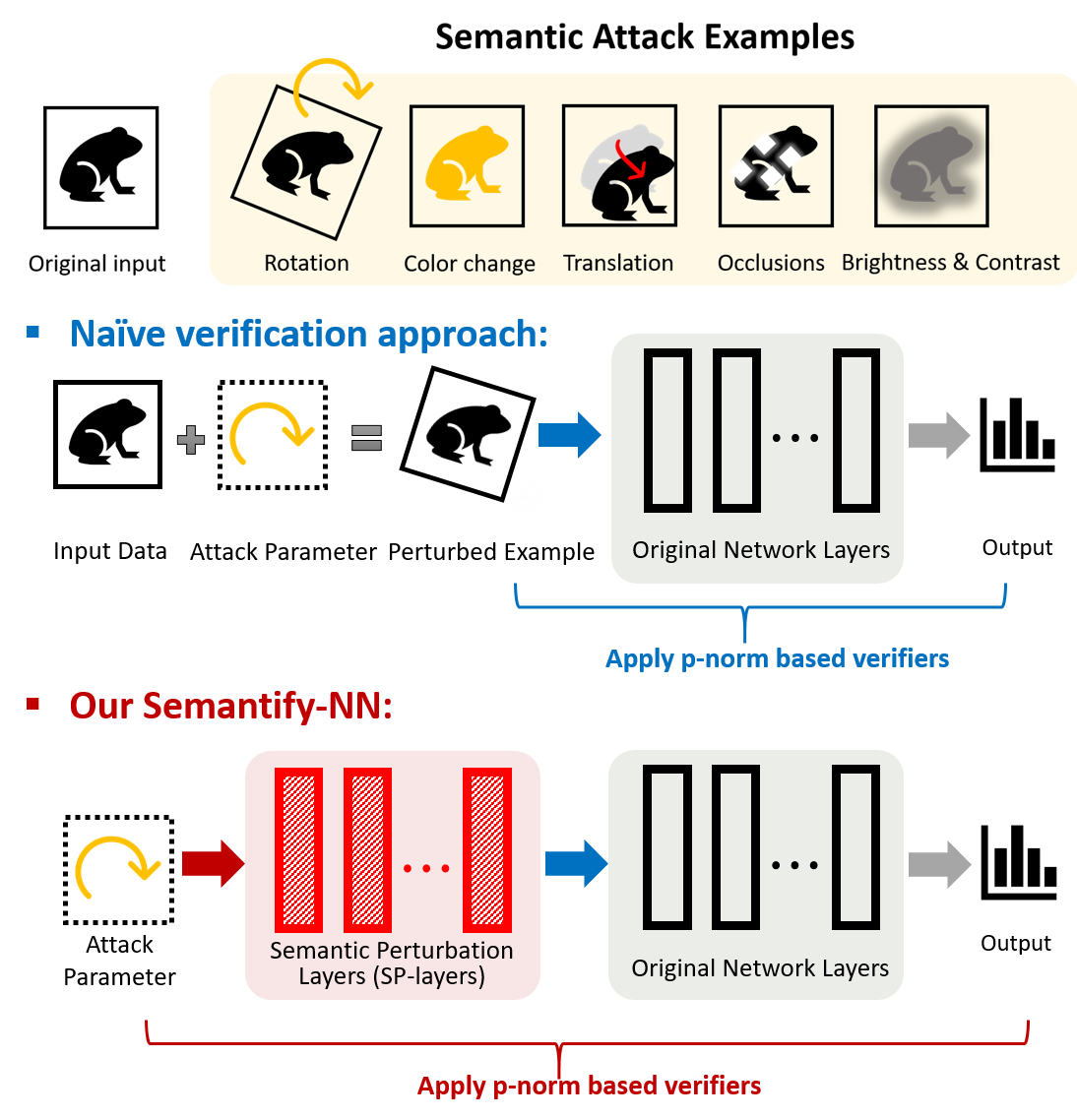}
    \centering
    \caption{Schematic illustration of our proposed \textit{Semantify-NN} robustness verification framework. Given a semantic attack threat model, \textit{Semantify-NN} designs the corresponding semantic perturbation layers (SP-layers) and inserts them to the input layer of the original network for verification. With SP-layers, \textit{Semantify-NN} can use any $\ell_p$-norm based verification method for verifying semantic perturbations.}
    \label{fig:veri_semantic}
\end{figure}

\section{Background and Related Works.}
\paragraph{$\ell_p$-norm based verification} For $\ell_p$-norm bounded threat models, current robustness verification methods are mainly based on solving a convex relaxation problem \cite{kolter2017provable,wong2018scaling,dvijotham2018dual}, devising tractable linear bounds on activation functions and layer propagation \cite{weng2018towards,zhang2018crown,wang2018efficient, singh2018boosting,Boopathy2019cnncert,singh2019abstract}. We refer readers to the prior arts and the references therein for more details. The work in \cite{wang2018efficient} considers brightness and contrast in the \textit{linear transformation} setting, which still falls under the $\ell_p$ norm threat model. The work in \cite{singh2019abstract} has scratched the surface of semantic robustness verification by considering rotations attacks with the $\ell_p$-norm based methods. However, we show that with our carefully designed refinement techniques, the robustness certificate can be significantly improved around 50-100\% in average. Moreover, we consider a more general and challenging setting than \cite{wang2018efficient} where the color space transformation can be \textit{non-linear} and hence directly applying $\ell_p$-norm based method could result in a very loose semantic robustness certificate. 
A recent work~\cite{balunovic2019certifying} considers geometric image transformations (i.e. spatial transformations such as rotation, translation, scaling, shearing), brightness and contrast, and interpolation, and their technique is based on sampling and requires solving a linear programming problem. In contrast, our framework is more general and efficient, because the main idea of our work is to describe non-linear semantic transformations with SP-Layers and apply a novel refinement technique. The idea of SP-Layers enables the use of any state-of-the-art $\ell_p$-norm based robustness verification tools~\cite{kolter2017provable,wong2018scaling,dvijotham2018dual,weng2018towards,zhang2018crown,wang2018efficient, singh2018boosting,Boopathy2019cnncert,singh2019abstract} and our proposed refinement technique has very light computational overhead on the linear-bounding based verifiers~\cite{weng2018towards,zhang2018crown,wang2018efficient, singh2018boosting,Boopathy2019cnncert,singh2019abstract} while obtaining significant performance improvement. The work in \cite{hamdi2019towards} proposes to use semantic maps to evaluate semantic robustness, but how to apply this analysis to develop semantic robustness verification is beyond their scope.

\paragraph{Semantic adversarial attack.}
In general, semantic adversarial attacks craft adversarial examples by tuning a set of parameters governing semantic manipulations of data samples, which are either explicitly specified (e.g. rotation angle) or implicitly learned (e.g. latent representations of generative models). In \cite{hosseini2018semantic}, the HSV (hue, saturation and value) representation of the RGB (red, green and blue) color space is used to find semantic adversarial examples for natural images. To encourage visual similarity, the authors propose to fix the value, minimize the changes in saturation, and fully utilize the hue changes to find semantic adversarial examples. In \cite{liu2018beyond}, the authors present a physically-based differentiable renderer allowing propagating pixel-level gradients to the parametric space of lightness and
geometry. In \cite{bhattad2019big}, the authors introduce texture and colorization to induce semantic perturbation with a large $\ell_p$ norm perturbation to the raw pixel space while remaining visual imperceptibility.
In \cite{joshi2019semantic}, an adversarial network composed of an encoder and a generator conditioned on attributes are trained to find semantic adversarial examples. In \cite{fawzi2015manitest,engstrom2017rotation}, the authors show that simple operations such as image rotation or object translation can result in a notable mis-classification rate. 

\paragraph{Continuously parameterized semantic perturbations cannot be enumerated.} \label{par:cont_sem} We note that for the semantic perturbations that are \textit{continuously parameterized} (such as Hue, Saturation, Lightness, Brightness, Contrast and Rotations that are considered in Section~\ref{sub:dis}), it is \textit{not} possible to enumerate all possible values even if we only perturb one single parameter (e.g. rotation angle). The reason is that these parameters take \textit{real values} in the continuous space, hence it is not possible to finitely enumerate all possible values, unlike its discrete parameterized counterpart (e.g. translations and occlusions have finite enumerations). Take the rotation angle for example, an attacker can try to do a grid search by sweeping rotation angle $\theta$ from $0^\circ$ to $90^\circ$ with a uniform $10^6$ grids. However, if the attacks are not successful at $\theta = 30^\circ$ and $30.00009^\circ$, it does not eliminate the possibility that there could exist some $\theta^\prime$ that could ``fool'' the classifier where $30^\circ< \theta^\prime < 30.00009^\circ$. This is indeed the motivation and necessity to have the robustness verification algorithm for semantic perturbations as proposed in this paper -- with a proper semantic robustness verification algorithm, we can \textit{guarantee} that neural networks will have a consistent prediction on the given image for any $\theta<a$, where $a$ is the semantic robustness certificate (of the image) delivered by our algorithm.

\begin{table*}[t!]
\centering
\caption{Semantic perturbations considered in Section~\ref{sec:theory}. All the listed attacks can be described in Equation~\eqref{eq:attack_para_dist} with $p = \infty$.}
\begin{tabular}{@{}c|cc|ccccc@{}}
\toprule
\multirow{2}{*}{\begin{tabular}[c]{@{}c@{}}Semantic \\ perturbation\end{tabular}} & \multicolumn{2}{c|}{Discretely parameterized} & \multicolumn{5}{c}{Continuously parameterized}                        \\ \cmidrule(l){2-8} 
                                                                                  & Translation            & Occlusion            & Hue      & Saturation & Lightness & Brightness \& Contrast & Rotation \\ \midrule
\# of attack parameter $k$   & 2    & 3   & 1  & 1   & 1     & 2    & 1        \\
SP-Layer    & \multicolumn{2}{c|}{Direct enumeration}   & Yes      & Yes        & Yes       & Yes   & Yes      \\ \bottomrule
\end{tabular}
\end{table*}
\section{Semantify-NN: A Semantic Robustness Verification Framework}
\label{sec:theory}
In this section, we formally introduce \textit{Semantify-NN}, a general robustness verification framework for semantic perturbations. We begin by providing a general problem formulation in Sec.~\ref{sec:formulation} and discuss the semantic perturbations that fall in this family in Sec.~\ref{sec:semantic_perturbation_details}. Next, we elucidate the two key components of Semantify-NN: \textit{SP-Layers} in Sec.~\ref{sec:semantic_perturbation_details} and a novel refinement technique \textit{Implicit Input Splitting} in Sec.~\ref{sub:refine}. These two components are the core of Semantify-NN to enable the use of any $\ell_p$-norm based verifiers and further boost verification performance with light computation overhead compared to existing refinement techniques.   

\subsection{Problem Formulation}
\label{sec:formulation}
Let $\Att$ be a general threat model. For an input data sample $x$, we define an associated space of perturbed images $x'$, denoted as the \textit{Attack Space $\Omega_{\Att}(x)$}, which is equipped with a distance function $\dA$ 
to measure the magnitude of the perturbation. The robustness certification problem under the threat model $\Att$ is formulated as follow: given a trained $K$-class neural network function $f$, input data sample $x$, we aim to find the largest $\delta$ such that 
\begin{align}
\label{eqn_formulation}
    \min_{x' \in \Omega_{\Att}(x), \dA(x', x) \leq \delta} \left( \min_{j \neq c} f_c(x') - f_j(x') \right) > 0,
\end{align}
where $f_j(\cdot)$ denotes the confidence (or logit) of the $j$-th class, $j \in \{1,2,\ldots,K\}$, and $c$ is the predicted class of un-perturbed input data $x$.

\subsection{Semantic Perturbation}
\label{sec:semantic_perturbation_details}
We consider semantic threat models that target semantically meaningful attacks, which are usually beyond the coverage of conventional $\ell_p$-norm bounded threat models in the pixel space. For an attack space $\Omega^k_\mathcal{A}$, there exists a function $g: X \times (I_1 \times I_2 \times \ldots \times I_k) \rightarrow X $ such that
    \begin{align} \label{eq:attack_para_dist}
        &\Omega^{k}_\mathcal{A}(x) = \{ g(x, \epsilon_1, \ldots, \epsilon_k) \mid \epsilon_i \in I_i\}, \\ \nonumber
        &d_{\mathcal{A}}(g(x, \epsilon_1, \ldots, \epsilon_k), x) = ||(\epsilon_1, \ldots, \epsilon_k)||_p, 
    \end{align}
where $X$ is the pixel space (the raw image), $I_i$ denotes a set of feasible semantic operations, and $\|\cdot\|_p$ denotes $\ell_p$ norm. The parameters $\epsilon_i$ specify semantic operations selected from $I_i$. For example, $\epsilon_i$ can describe some human-interpretable characteristic of an image, such as translations shift, rotation angle, etc. For convenience we define $\E^k = (\epsilon_1, \epsilon_2, \ldots, \epsilon_k)$ and $I^k = I_1 \times \ldots \times I_k$ where $k$ denotes the dimension of the semantic attack. In other words, we show that it is possible to define an explicit function $g$ for all the semantic perturbations considered in this work, including translations, occlusions, color space transformations, and rotations, and we then measure the $\ell_p$ norm of the semantic perturbations on the space of semantic features  $\E^k$ rather than the raw pixel space. Notice that the conventional $\ell_p$ norm perturbations on the raw RGB pixels is a special case under this definition: by letting $I_i$ equal to a bounded real set (i.e. $x'_i-x_i$, all possible difference between $i$-th pixel) and $k$ be the dimension of input vector $x$, we recover $\dA = \|x'-x\|_p$. 

Based on the definition above, semantic attacks can be divided into two categories: discretely parameterized perturbations (i.e. $I_k$ is a discrete set) including translation and occlusions and continuously parameterized perturbations (i.e. $I_k$ is a continuous set) including color changes, brightness, contrast, and spatial transformations (e.g. rotations).

\subsubsection{Discretely Parameterised Semantic Perturbation}
\label{sub:dis}

\paragraph{Translation} Translation is a $2$-dimensional semantic attack with the parameters being the relative position of left-uppermost pixel of perturbed image to the original image. Therefore, $I_1 = \{0, 1, \ldots, r\}, I_2 = \{0, 1, \ldots, t\}$ where $r,t$ are the dimensions of width and height of our input image $x$. Note that any padding methods can be applied including padding with the black pixels or boundary pixels, etc. 

\paragraph{Occlusion} Similar to translation, occlusion attack can be expressed by $3$-dimensional attack parameters: the coordinates of the left-uppermost pixel of the occlusion patch and the occlusion patch size\footnote{we use squared patch, but it can be rectangular in general.}. Note that for discretely paramerterised semantic perturbations, provided with sufficient computation resources, one could simply exhaustively enumerate all the possible perturbed images. At the scale of our considered image dimensions, we find that exhaustive enumeration can be accomplished within a reasonable computation time and the generated images can be used for direct verification. In this case, the SP-layers are reduced to enumeration operations given a discretely parameterized semantic attack threat model. Nonetheless, the computation complexity of exhaustive enumeration grows combinatorially when considering a joint attack threat model consisting of multiple types of discretely parameterized semantic attacks.

\subsubsection{Continuously Parameterized Semantic Perturbation}
\label{sub:cont}
Most of the semantic perturbations fall under the framework where the parameters are continuous values, i.e., $I^k \subset \rr^k$. We propose the idea of adding \textit{semantic perturbation layers} (SP-layers) to the input layer of any given neural network model for efficient robustness verification, as illustrated in Figure \ref{fig:veri_semantic}. By letting $g_x(\E^k) = g(x, \E^k)$, the verification problem for neural network $f$ formulated in \eqref{eqn_formulation} becomes
\begin{equation}
   \min_{\E^k \in I^k, \dA(g_x(\E^k) ,x) \leq \delta} \left(\min_{j \neq c} f_c(g_x(\E^k)) - f_j(g_x(\E^k)) \right) > 0.  
\end{equation}

If we consider the new network as $f^{sem} = f \circ g_x$, then we have the following problem:
\begin{equation}
    \min_{\E^k \in I^k, ||\E^k||_p \leq \delta} \left(\min_{j \neq c}  f_c^{sem}(\E^k) - f_j^{sem}(\E^k) \right) > 0, 
\end{equation}
which has a similar form to $\ell_p$-norm perturbations but now on the semantic space $I^k$. The proposed SP-layers allow us to explicitly define the dimensionality of our perturbations and put explicit dependence between the manner and the effect of the semantic perturbation on different pixels of the image. In other words, one can view our proposed SP-layers as a parameterized input transformation function from the semantic space to RGB space and $g(x, \E^k)$ is the perturbed input in the RGB space which is a function of perturbations in the semantic space. Our key idea is to express $g$ in terms of commonly-used activation functions and thus $g$ is in the form of neural network and can be easily incorporated into the original neural network classifier $f$. Note that $g$ can be arbitrarily complicated to allow for general transformations for SP-layers; nevertheless, it does not result in any difficulties to apply the conventional $\ell_p$-norm based methods such as \cite{zhang2018crown,wang2018efficient,singh2018fast,Boopathy2019cnncert,weng2018towards}, as we only require the activation functions to have custom linear bounds and do not need them to be continuous or differentiable. Below we specify the explicit form of SP-layers corresponding to five different semantic perturbations using (i) hue, (ii) saturation, (iii) lightness, (iv) brightness and contrast, and (v) rotation. 

\paragraph{Color space transformation}
We consider color transformations parameterized by the hue, saturation and lightness (HSL space). Unlike RGB values, HSL form a more intuitive basis for understanding the effect of the color transformation as they are semantically meaningful. For each of the basis, we can define the following functions for $g$:
\begin{itemize}[leftmargin=6mm]
    \item \textit{Hue} \quad 
    This dimension corresponds to the position of a color on the color wheel. Two colors with the same hue are generally considered as different shades of a color, like blue and light blue. The hue is represented on a scale of $0$-$360^\circ$ which we have rescaled to the range $[0,6]$ for convenience. Therefore, we have  $g(R,G,B,\epsilon_h) = (d\cdot \phi^h_R(h')+m, d\cdot \phi^h_G(h')+m, d\cdot \phi^h_B(h')+m)$, where $d = (1-|2l-1|)s, m = l - \frac{d}{2}$ and $h' = (h + \epsilon_h)\mod{6}$ are functions of $R,G,B$ independent of $\epsilon_h$ and 
    \begin{equation*}
        \fontsize{9pt}{10pt}
      (\phi^h_R(h'),\phi^h_G(h'),\phi^h_B(h'))  =
        \begin{cases}
          (1, V, 0)& 0 \leq h' \leq 1\\
          (V, 1, 0)& 1 \leq h' \leq 2\\
          (0, 1, V)& 2 \leq h' \leq 3\\
          (0, V, 1)& 3 \leq h' \leq 4\\
          (V, 0, 1)& 4 \leq h' \leq 5\\
          (1, 0, V)& 5 \leq h' \leq 6\\
        \end{cases}     
    \end{equation*}
    where $V = (1 - |(h'\mod{2}) - 1|) $.
    
    For $0 \leq h' \leq 6$ the above can be reduced to the following in the ReLU form ($\sigma_i(x) =$ ReLU$(x-i)$) and hence can be seen as one hidden layer with ReLU activation connecting from hue space to original RGB space:
    \begin{equation}
    \fontsize{9pt}{10pt}
    \begin{split}
        \phi^h_R(h') &= 1 + \sigma_2(h') + \sigma_4(h') - (\sigma_5(h') + \sigma_1(h')) \\
        \phi^h_G(h') &= \sigma_0(h') + \sigma_4(h') - (\sigma_1(h')  + \sigma_3(h')) \\
        \phi^h_B(h') &= \sigma_2(h')  + \sigma_6(h') - (\sigma_5(h') + \sigma_3(h')) \\
    \end{split}
    \end{equation}
    \item \textit{Saturation} \quad
    This corresponds to the colorfulness of the picture. At saturation $0$, we get grey-scale images; while at a saturation of $1$, we see the colors pretty distinctly. We have $g(R,G,B,\epsilon_s) = (d_R\cdot \phi^s(s') + l, d_G\cdot \phi^s(s') + l, d_B\cdot \phi^s(s') + l) $ where $s' = s + \epsilon_s, d_R = \frac{R - l}{s}, d_G = \frac{G - l}{s}$ and $d_B = \frac{B - l}{s}$ are functions of $R,G,B$ independent of $\epsilon_s$ and
    \begin{equation}
        \fontsize{9pt}{10pt}
        \phi^s(s') = \min(\max(s', 0), 1) = \sigma_0(s') - \sigma_1(s')
    \end{equation}
    \item \textit{Lightness} \quad
    This property corresponds to the perceived brightness of the image where a lightness of $1$ gives us white and a lightness of $0$ gives us black images. In this case, 
    $g(R,G,B,\epsilon_l) = (d_R\cdot \phi_1^l(l')+\phi_2^l(l'), d_G\cdot \phi_1^l(l')+\phi_2^l(l'), d_B\cdot \phi_1^l(l')+\phi_2^l(l')) $ where $l' = l+\epsilon_l$, $d_R = \frac{R - l}{1 - |2l -1|},d_G = \frac{G - l}{1 - |2l -1|} $, and $d_B=\frac{B - l}{1 - |2l -1|}$ are functions of $R,G,B$ independent of $\epsilon_l$ and
    \begin{equation}
        \fontsize{9pt}{10pt}
        \begin{split}
            \phi_1^l(l') &= 1 - |2\cdot \min(\max(l',0),1) - 1| \\
            &= -\sigma_0(2\cdot l') - \sigma_2(2\cdot l') + 2\cdot\sigma_1(2\cdot l') +1\\
            \phi_2^l(l') &= \min(\max(l', 0), 1) =  \sigma_0(l') - \sigma_1(l')
        \end{split}
    \end{equation}
\end{itemize}
\paragraph{Brightness and contrast}
We also use the similar technique as HSL color space for multi-parameter transformations such as brightness and contrast: the attack parameters are $\epsilon_b$ for brightness perturbation and $\epsilon_c$ for contrast perturbation, and we have
\begin{align}
    g(x, \epsilon_b, \epsilon_c) &=  \min(\max((1+\epsilon_c)\cdot x + \epsilon_b,0),1) \\
    &= \sigma_0((1+\epsilon_c)\cdot x + \epsilon_b) -  \sigma_1((1+\epsilon_c)\cdot x + \epsilon_b) \nonumber 
\end{align}
Therefore, $g$ can be expressed as one additional ReLU layer before the original network model, which is the proposed SP Layers in Figure~\ref{fig:veri_semantic}. 

\paragraph{Rotation}
We have $1$-dimensional semantic attack parameterized by the rotation angle $\theta$, and we consider rotations at the center of the image with the boundaries being extended to the area outside the image. We use the following interpolation to get the values $x'_{i,j}$ of output pixel at position $(i,j)$ after rotation by $\theta$. Let $i' = i\cos\theta - j\sin\theta, j' = j\cos\theta + i\sin\theta$, then
\begin{equation}
    \fontsize{9pt}{10pt}
   x'_{i,j} = \frac{\sum_{k,l} x_{k,l} \cdot \max(0, 1 - \sqrt{(k-i')^2 + (l-j')^2})}{\sum_{k,l} \max(0, 1 - \sqrt{(k-i')^2 + (l-j')^2})} 
\end{equation}
where $k,l$ range over all possible values. For individual pixels at position $(k,l)$ of the original image, the scaling factors for its influence on the output pixel at position $(i,j)$ is given by the following function:
\begin{equation}
    \fontsize{9pt}{10pt}
 m_{(k,l), (i,j)}(\theta) = \frac{\max(0, 1 - \sqrt{(k-i')^2 + (l-j')^2})}{\sum_{k',l'} \max(0, 1 - \sqrt{(k'-i')^2 + (l'-j')^2})} 
\end{equation}
which is highly non-linear. It is $0$ for most $\theta$ and for a very small range of $\theta$, it takes non-zero values which can go up to $1$. Thus, it makes naive verification infeasible. One idea is to use \textit{Explicit Input Splitting} in Sec. \ref{sub:refine} to split the input the range of $\theta$ into smaller parts and certify all parts, which will give a tighter bound since in smaller ranges the bounds are tighter. However, the required number of splitsin \textit{Explicit Input Splitting} may become too large, making it computationally infeasible. To balance this trade-off, we propose a new refinement technique named as \textit{implicit input splitting} in the following section, which has light computational overhead and helps substantial boost in verification performance.

\subsection{Input Space Refinement for Semantify-NN}
\label{sub:refine}
To better handle highly non-linear functions that might arise from the general activation functions in the SP-layers, we propose two types of input-level refinement strategies. For linear-relaxation based verification methods, the following Theorem holds with the proof in the appendix.
\begin{theorem}\label{thm: Convex hull}
    Given an image $x$, if we can verify that a set $S$ of perturbed images $x'$ is correctly classified for a threat model using one certification cycle, then we can verify that every perturbed image $x'$ in the convex hull of $S$ is also correctly classified, where the convex hull in the pixel space.
\end{theorem}
Here one certification cycle means one pass through the certification algorithm sharing the same linear relaxation values. Although $\ell_p$-norm balls are convex regions in pixel space, other threat models (especially semantic perturbations) usually do not have this property. This in turn poses a big challenge for semantic verification. 
\begin{remark}\label{rem: Necessity}
   For some non-convex attack spaces embedded in high-dimensional pixel spaces, the convex hull of the attack space associated with an image can contain images belonging to a different class (an example of rotation is illustrated in Figure \ref{fig: convex_hull} in the appendix). Thus, one cannot certify large intervals of perturbations using a single certification cycle of linear relaxation based verifiers.
\end{remark}

\paragraph{Explicit Input Splitting}

As we cannot certify large ranges of perturbation simultaneously, input-splitting is essential for verifying semantic  perturbations. It reduces the gap between the linear bounds on activation functions and yields tighter bounds, as illustrated in appendix. We observe that
\begin{equation}
    \begin{split}
        \min_{\E^k \in (I_1^k \cup I_2^k), h(\E^k) \leq \delta} \min_{j \neq c} f_c^{sem}(\E^k) - f_j^{sem}(\E^k)  = \\
        \min_{l \in \{1,2\}} \min_{\E^k \in I_l^k, h(\E^k) \leq \delta}  \min_{j \neq c} f_c^{sem}(\E^k) - f_j^{sem}(\E^k) \nonumber
    \end{split}
\end{equation}

If  ${\small{\min_{\E^k \in I_l^k, h(\E^k) \leq \delta}(\min_{j} f_c^{sem}(\E^k) - f_j^{sem}(\E^k)) > 0 } }$ holds for both parameter dimensions $l = \{1,2\}$, then we have  $ \min_{\E^k \in I_1^k \cup I_2^k, h(\E^k) \leq \delta}(\min_{j} f_c^{sem}(\E^k) - f_j^{sem}(\E^k)) > 0 $. As a result, we can split the original interval into smaller parts and certify each of them separately in order to certify the larger interval. The drawback of this procedure is that the computation time scales linearly with the number of divisions as one has to run the certification for every part. However, for the color space experiments we find that a few partitions are already sufficient for tight certificate.

\paragraph{Implicit Input Splitting}
As a motivating example, in Figure \ref{fig:awesome_image1} in the appendix, we give the form of the activation function for rotation. Even in a small range of rotation angle $\theta$ ($2^\circ$), the function is quite non-linear resulting in very loose linear bounds. As a result, we find that we are unable to get good verification results for datasets like MNIST and CIFAR-10 without increasing the number of partitions to very large values ($\approx 40,000$). This makes verification methods computationally infeasible. We aim at reducing the cost in \textit{explicit splitting} by combining the intermediate bounds used by linear relaxation methods (\cite{raghunathan2018semidefinite}-\cite{zhang2018crown}) to compute the suitable relaxation for the layer-wise non-linear activations. The idea is to use the shared linear bounds among all sub-problems, and hence we only need to construct the matrices (Def 3.3-Cor 3.7 in \cite{weng2018towards}) $\mathbf{A}^{(k)}, \mathbf{T}^{(k)}, \mathbf{H}^{(k)}$ once for all $S$-subproblems instead of having different matrices for each subproblem. This helps to reduce the cost significantly from a factor of $S$ to $1$ when $S$ is large (which is usually the case in order to get good refinement). 

For the implementation, we split the original problem into $S$ subproblem. To derive bounds on the output of a neuron at any given layer $l$, we calculate the pre-activation range for every subproblem. Then we merge the intervals of each neuron among all the subproblems (e.g. set $\bm{u}_{r} = \max_i \bm{u}_{r, \textrm{sub} i}, \bm{l}_{r} = \min_i \bm{l}_{r, \textrm{sub} i}$ in \cite{weng2018towards}) to construct the linear relaxation, $\mathbf{A}^{(k)}, \mathbf{T}^{(k)}, \mathbf{H}^{(k)}$  for the post-activation output of layer $l$. Continuing this procedure till the last layer gives the bounds on the output of the whole neural network.

\paragraph{Remark.} Our experiments demonstrate that, for semantic perturbations, the refinement on the input space of the semantic space can significantly boost up the tightness of robustness certificate, as shown in Section \ref{sec_exp}. Although additional improvement can be made by refining the pre-activation bounds of each layer through solving a Linear Programming problem or Mixed integer optimization problem similar to the works in~\cite{wang2018efficient} and~\cite{singh2018boosting} for the $\ell_p$-norm input perturbation, we observed that our proposed approach with input space refinement has already delivered a certified lower bound (of the minimal successful semantic perturbation) very close to the attack results (which is an upper bound), suggesting that the feature layer refinement will only have minor improvement while at the cost of much larger computation overhead (grows exponentially with number of nodes to be refined).

\begin{table*}[t!]
  \centering
  \captionsetup{justification=centering}
  \caption{Evaluation of averaged bounds on HSL space perturbation. SPL denotes our proposed SP-layers. SPL + Refine refers to certificate obtained after using explicit splitting. Grid search on parameter space is used for attack. The results demonstrate the significance of using SPL layers for certification.}
  \begin{adjustbox}{max width=0.9\textwidth}
    \begin{tabular}{|l|cccc||cc||c|}
    \hline
    
    Network & \multicolumn{4}{c||}{Certified Bounds} & \multicolumn{2}{c||}{Ours Improvement (vs Weighted)}  & Attack  \\
    \hline
          &  Naive & Weighted & \bf SPL & \bf SPL + Refine & w/o refine & w/ refine  & Grid  \\
    \hline 
    \addlinespace[0.1em]
    
    \multicolumn{8}{|l|}{\bf Experiment (I)-A: Hue}\\
    \addlinespace[0.1em]
    \hline 
    CIFAR, MLP 6 $\times$ 2048  & 0.00316 & 0.028 & 0.347   & 0.974 & 11.39x & 51.00x & 1.456   \\
    CIFAR, CNN 5 $\times$ 10    & 0.0067 & 0.046 & 0.395   & 1.794 & 7.58x  & 38.00x & 1.964   \\
    GTSRB, MLP 4 $\times$ 256 & 0.01477	& 0.091	& 0.771	& 2.310 & 8.47x & 22.31x & 2.388 \\ 
    GTSRB MLP 4 $\times$ 256 sem adv & 0.01512 & 0.092 & 0.785 & 2.407 & 8.53x & 26.16x & 2.474 \\
    \hline
    \addlinespace[0.1em]
    \multicolumn{8}{|l|}{\bf Experiment (I)-B: Saturation}\\
    \addlinespace[0.1em]
    \hline 
    CIFAR, MLP 6 $\times$ 2048  & 0.00167 & 0.004 & 0.101   & 0.314 & 24.25x  & 77.50x & 0.342   \\
    CIFAR, CNN 5 $\times$ 10    & 0.00348 & 0.019 & 0.169   & 0.389 & 7.89x   & 19.47x & 0.404   \\
    GTSRB, MLP 4 $\times$ 256 & 0.00951	& 0.020 & 0.38	& 0.435 & 19.00x & 21.75x & 0.444 \\
    GTSRB MLP 4 $\times$ 256 sem adv & 0.00968	& 0.020	& 0.431 & 0.458 & 21.55x & 22.90x & 0.467 \\
    \hline
    
    \addlinespace[0.1em]
    \multicolumn{8}{|l|}{\bf Experiment (I)-C: Lightness}\\
    \addlinespace[0.1em]
    \hline 
    CIFAR, MLP 6 $\times$ 2048  & 0.00043 & 0.001 & 0.047   & 0.244  & 46.00x  & 243.00x   & 0.263  \\
    CIFAR, CNN 5 $\times$ 10    & 0.00096 & 0.002 & 0.080   & 0.273  & 39.00x  & 135.50x   & 0.303  \\
    GTSRB, MLP 4 $\times$ 256 & 0.0025	& 0.005 & 0.134	& 0.332 & 26.80x & 66.40x & 0.365 \\
    GTSRB MLP 4 $\times$ 256 sem adv &0.00268	& 0.005	& 0.148	& 0.362	 & 29.80x & 72.40x & 0.398 \\
    \hline
    \end{tabular}%
    \end{adjustbox}
  \label{tab:hsl}%
\end{table*}

\section{Experiments}
\label{sec_exp}
We conduct extensive experiments for all the continuously-parametrized semantic attack threat models presented in the paper. The verification of discretely-parametrized semantic perturbations can be straightforward using enumeration, as discussed in Section \ref{sub:dis}.
By applying our proposed method, Semantify-NN, one can leverage $L_p$-norm verification algorithms including \cite{weng2018towards,zhang2018crown,wang2018efficient,singh2018fast,Boopathy2019cnncert}. We use verifiers proposed in \cite{zhang2018crown} and \cite{Boopathy2019cnncert} to certify multilayer perceptron (MLP) models and convolutional neural network (CNN) models as they are open-sourced, efficient, and support general activations on MLP and CNN models.   
\begin{itemize}[leftmargin=*]
    \item \textbf{Baselines}. We calculate the upper bound and lower bound for possible value ranges of each pixel $x_i$ of the original image given perturbation magnitude in the semantic space. Then, we use $L_{\infty}$-norm based verifier to perform bisection on the perturbation magnitude and report its value. It is shown that directly converting the perturbation range from semantic space to original RGB space and then apply $L_p$-norm based verifiers give very poor results in all Tables. We also give a weighted-eps version where we allow for different levels of perturbation for different pixels.
    
    \item \textbf{Attack}. We use a grid-search attack with the granularity of the order of the size of the sub-intervals after input splitting. Although this is not the optimal attack value, it is indeed an upper bound for the perturbation. Increasing the granularity would only result in a tighter upper bound and does not affect the lower bound (the certificate we deliver). We would like to highlight again that even though the threat models are very low dimensional, they are continuously parametrized and cannot be certified against by enumeration as discussed in Sec \ref{par:cont_sem}.
    
    \item \textbf{Semantify-NN}: 
    We implement both SP-layers (\textbf{SPL}) and with refinement (\textbf{SPL+Refine}) described in Section~\ref{sub:refine}.
    
\end{itemize}

\noindent \textbf{Implementations, Models and Datasets.} In all of our experiments, we use a custom google cloud instance with 24 vCPUs (Intel Xeon CPU @ 2.30GHz) and 90GB RAM. The SP-layers are added as fully connected layers for MLP's and as modified convolution blocks for CNN models (we allow the filter weights and biases to be different for different neurons). We evaluate Semantify-NN and other methods on MLP and CNN models trained on the MNIST, CIFAR-10 and GTSRB (German Traffic Sign Benchmark datasets). We use the MNIST and CIFAR models released by [6] and their standard test accuracy of MNIST/CIFAR models are $98$-$99$\%/$60$-$70$\%. We train the GTSRB models from scratch to have $94$-$95$\% test accuracies. All CNNs (LeNet) use 3-by-3 convolutions and two max pooling layers and along with filter size specified in the description for two convolution layers each. LeNet uses a similar architecture to LeNet-5~\cite{lecun1998gradient}, with the no-pooling version applying the same convolutions over larger inputs. We also have two kinds of adversarially trained models. The models (denoted as sem adv in the Table) are trained using data augmentation where we add perturbed images (according to the corresponding threat model) to the training data. The models denoted as $l_\infty$ adv are trained using the $L_\infty$ norm adversarial training method \cite{madry2017towards}. We evaluate all methods on 200 random test images and random targeted attacks. We train all models for 50 epochs and tune hyperparameters to optimize validation accuracy.

\begin{table*}[!htbp]
  \centering
  \captionsetup{justification=centering}
  \caption{Evaluation of averaged bounds on rotation space perturbation. SPL denotes our proposed SP-layers. The certified bounds obtained from SPL+Refine are close to the upper bounds from grid attack.}
  \begin{adjustbox}{max width=0.8\textwidth}
    \begin{tabular}{|l|ccc|c||c|}
    \hline
    
    Network & \multicolumn{4}{c||}{Certified Bounds (degrees)}  & Attack (degrees)\\
    \hline
     & \multicolumn{3}{c|}{Number of Implicit Splits}  & \bf SPL + Refine & Grid Attack \\
    \hline
          &  \shortstack{1 implicit \\ No explicit}& \shortstack{5 implicit \\ No explicit} &  \multicolumn{1}{c|}{\shortstack{10 implicit \\ No explicit}}  & \shortstack{100 implicit + \\ \newline explicit intervals of $0.5^\circ$} &\\
    \hline 
    \addlinespace[0.1em]
    \multicolumn{6}{|l|}{\bf Experiment (II): Rotations} \\
    \addlinespace[0.1em]
    \hline 
    MNIST, MLP 2$\times$ 1024   & 0.627 & 1.505 & 2.515 & 46.24  & 51.42 \\
    MNIST, MLP 2$\times$ 1024 $l_\infty$ adv & 1.376 & 2.253 & 2.866 & 45.49  & 46.02 \\
    MNIST, CNN LeNet            & 0.171 & 0.397 & 0.652 & 43.33 & 48.00 \\
    CIFAR, MLP 5 $\times$ 2048  & 0.006 & 0.016 & 0.033 & 14.81 & 37.53 \\
    CIFAR, CNN 5 $\times$ 10    & 0.008 & 0.021 & 0.042 & 10.65  & 30.81 \\
    GTSRB, MLP 4 $\times$ 256   & 0.041 & 0.104 & 0.206 & 31.53  & 33.43 \\
    \hline
    \end{tabular}%
    \end{adjustbox}
  \label{tab: rotation}%
\end{table*}%

\paragraph{Experiment (I): HSL Space Perturbations.}
Table \ref{tab:hsl} demonstrates that using $L_p$-norm based verification results in extremely loose bounds because of the mismatch in the dimensionality of the semantic attack and dimensionality of the induced $L_p$-norm attack. Explicitly introducing this dimensionality constraint by augmenting the neural networks with our proposed  SP-layers gives a significant increase in the maximum certifiable lower bound, resulting in $4-50 \times$ larger bounds. However, there is still an apparent gap between the Semantify-NN's certified lower bound and attack upper bound. Notably, we observe that adding input-space refinements helps us to further tighten the bounds, yielding an extra $1.5-5 \times$ improvement. This corroborates the importance of input splitting for the certification against semantic attacks. The transformations for HSL space attacks are fairly linear, so the gap between our certified lower bound and attack upper bound becomes quite small.

\paragraph{Experiment (II): Rotation}
Table \ref{tab: rotation} shows the results of rotation space verification. Rotation induces a highly non-linear transformation on the pixel space, so we use this to illustrate the use of refinement for certifying such functions. As the transforms are very non-linear, the linear bounds used by our SP-layers are very loose, yielding very small robustness certification. In this case, explicit input splitting is not a computationally-appealing approach as there are a huge amount of intervals to be certified. Table \ref{tab: rotation} shows how using implicit splits can increase the size of certifiable intervals to the point where the total number of intervals needed is manageably big. At this point we use explicit splitting to get tight bounds. For the results in \textbf{SPL + Refine}, we use intervals of size $0.5$ at a time with 100 implicit splits for each interval. The effect on the number of implicit splitting is discussed in Appendix~\ref{app:implicit_splitting}.   

\begin{table}
    \centering
  \captionsetup{justification=centering}
  \caption{Evaluation of averaged bounds and run time on translation space perturbation and occlusions.}
  \begin{adjustbox}{max width=\linewidth}
    \begin{tabular}{|l|cc|cc|}
    \hline
     & \multicolumn{2}{c|}{Translation}  &  \multicolumn{2}{c|}{Occlusion}  \\
    \hline
    Network & \shortstack{Number of  \\ Pixels} &  \shortstack{Runtime \\(sec)} & \shortstack{Filter \\ Size} & \shortstack{Runtime \\(sec)} \\
    \hline 
    \addlinespace[0.1em]
    \multicolumn{5}{|l|}{\bf Experiment (III): Translation \& Occlusion} \\
    \addlinespace[0.1em]
    \hline 
    MNIST, MLP 3$\times$ 1024   & 4.063 & 0.46 & 11.080 & 0.44 \\
    MNIST, MLP 4$\times$ 1024   & 4.130 & 0.51 & 10.778 & 0.55 \\
    MNIST, CNN LeNet            & 4.852 & 0.44 & 10.560 & 0.64 \\
    MNIST, CNN 4 $\times$ 5     & 5.511 & 0.47 & 11.590 & 0.71 \\
    CIFAR, MLP 5 $\times$ 2048  & 3.251 & 1.22 & 3.264 & 1.11 \\
    CIFAR, MLP 6 $\times$ 2048  & 2.755 & 1.36 & 3.300 & 1.28 \\
    CIFAR, CNN 5 $\times$ 10    & 4.532 & 0.86 & 3.373 & 0.66 \\
    GTSRB, MLP 4 $\times$ 256   & 2.785 & 0.39 & 2.652 & 0.14 \\
    \hline
    \addlinespace[0.1em]
    \end{tabular}%
    \end{adjustbox}
  \label{tab: discrete_transform}%
\end{table}%

\paragraph{Experiment (III): Brightness and Contrast}

For multi-dimensional semantic attacks (here a combination attack using both brightness and contrast), we can consider any $L_p$ norm of the parameters to be our distance function. In Figure \ref{fig: barplot_brightness} we show the results for average lower bound for brightness perturbations while fixing the maximum perturbation for contrast parameter ($\epsilon$) to be 0.01, 0.03 and 0.05.

\begin{figure}[ht]
 \centering
 \includegraphics[width=\linewidth]{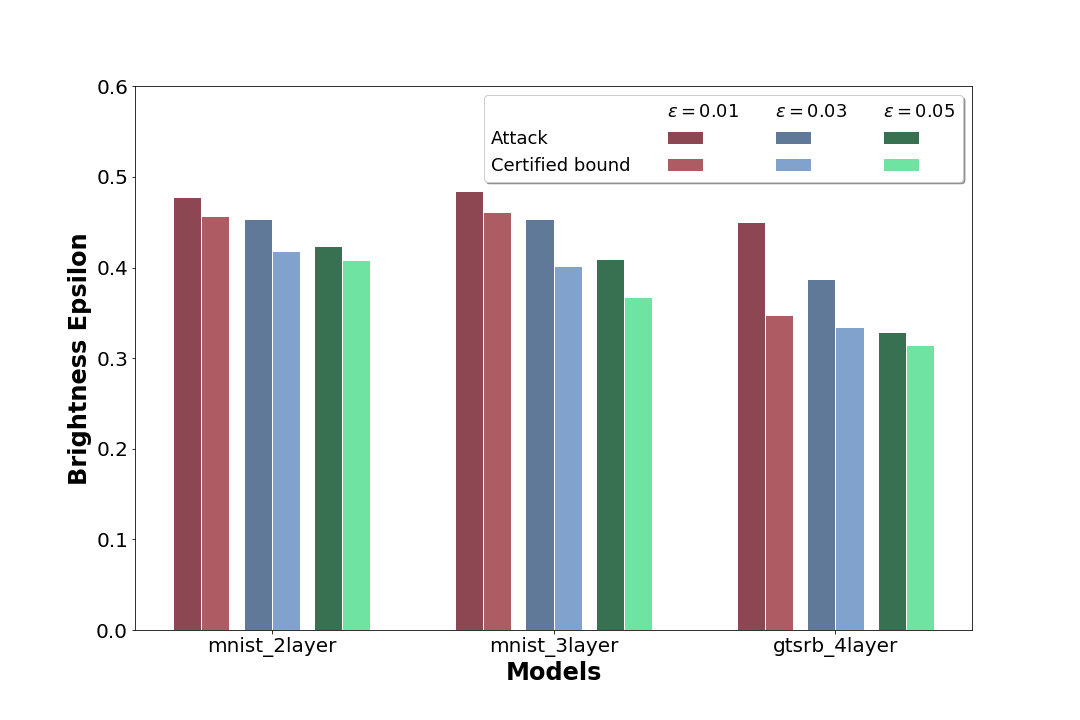}
 \caption{Semantic certification for brightness and contrast.}
 \label{fig: barplot_brightness}
\end{figure}

\paragraph{Experiment (IV): Translation and Occlusion}

Table \ref{tab: discrete_transform} shows that in these experiments, the run-times needed for translation and occlusion to perform exhaustive enumeration are quite affordable. For translation we consider the attack space with $I_1 = \{0, 1, \ldots, r\}$ and $I_2 = \{0, 1,  \ldots, t\}$ where the images are of shape $r \times c$. The reported values are the average $L_2$ norm values for $(\epsilon_1, \epsilon_2)$, the shift vector. For occlusion we consider the attack space with $I_1 = \{0, 1, \ldots, r\}$, $I_2 = \{0, 1,  \ldots, t\}$ and $I_3 = \{1, 2, 3, \ldots, \max\{r,t\}\}$, where the images are of shape $r \times t$. The reported values are the average $L_\infty$ norm values for $\epsilon_3$ (the side-length of the square patch).

\section{Conclusion}
In this paper, we propose \textit{Semantify-NN}, a model-agnostic and generic verification framework for neural network image classifiers against a broad spectrum of semantic perturbations. Semantify-NN exclusively features semantic perturbation layers (SP-layers) to expand the verification power of current verification methods beyond $\ell_p$-norm bounded threat models. Based on a diverse set of semantic attacks, we demonstrate how the SP-layers can be implemented and refined for verification. Evaluated on various datasets, network architectures and semantic attacks, our experiments corroborate the effectiveness of \textit{Semantify-NN} for semantic robustness verification.

{\small
\bibliographystyle{ieee_fullname}
\bibliography{ref}
}

\newpage

\clearpage
\begin{appendices}

\section{Proof of Input Refinement}

\begin{theorem}\label{thm: Convex hull_2}
    If we can verify that a set $S$ of perturbed versions of an image $x$ are correctly classified for a threat model using one \textbf{certification cycle} (one pass through the algorithm sharing the same linear relaxation values), then we can verify that every perturbed image in the convex hull of $S$ is also correctly classified, where we take the convex hull in the pixel space.
\end{theorem}
\begin{proof}
    When a set $S$ of perturbed inputs and a neural network $\mathrm{f}_{NN}$ are passed into a verifier, it produces $A_L,b_L, A_U, b_U$ such that for all $y \in S$
    \begin{equation}
        A_L\cdot y + b_L \leq \mathrm{f}_{NN}(y)_j\leq A_U \cdot y + b_U
    \end{equation}
\begin{claim}
    We claim that if $y, z \in S$, then $x= \frac{y + z}{2}$ satisfies the above inequality. 
\end{claim}
\begin{proof}
    We can prove this by induction on the layers. For the first layer we see that as matrix multiplication and addition are linear transformations, we have that $x_1 = W_1 \cdot x + b_1$ lies between the points $y_1 = W_1 \cdot y + b_1$ and $z_1 = W_1 \cdot z + b_1$. The important property to note here is that every co-ordinate of $x_1$ lies in the interval between the co-ordinates of $y_1$ and $z _1$. Now, we see that the activation layer is linear relaxed such that $A^1_L\cdot y + B^1_L\leq Act(y) \leq A^1_U \cdot y + B^1_U$ for all values of $y$ between the upper and lower bound for a neuron. As we proved before every pixel of $x$ lies within the bounds and hence satisfies the relation.
    
    For the inductive case, we see that given that $x$ satisfies this relation up till layer $l$, then we have that 
    \begin{equation}
        A^l_L\cdot x + b^l_L \leq \mathrm{f}^l_{NN}(x)_j\leq A^l_U \cdot x + b^l_U
    \end{equation}
    where $\mathrm{f}^l_{NN}(x)_j$ gives the output of the $j^{th}$ neuron in layer $l$ post-activation.
    
    Now, we see that as we satisfy the above equation, the certification procedures ensure that the newly computed pre-activation values satisfy the same condition. So, we have
    $$A^{l+1/2}_L\cdot x + b^{l+1/2}_L \leq \mathrm{f}^{l+1/2}_{NN}(x)_j\leq A^{l+1/2}_U \cdot x + b^{l+1/2}_U$$
    
    where we use $l+1/2$ to denote the fact that this is a pre-activation bound. Now, if we can show that our value lies within the range of the output of every neuron, then we prove the inductive case. But then we see that as these $A^{l+1/2}_L\cdot x + b^{l+1/2}_L$ is a linear transform $x_{l+1/2} = A^{l+1/2}_L\cdot x + b^{l+1/2}_L$ lies between the points $y_{l+1/2} = A^{l+1/2}_L\cdot y + b^{l+1/2}_L$, $z_{l+1/2} = A^{l+1/2}_L\cdot z + b^{l+1/2}_L$. So, we see that the values taken by this is lower bounded by the corresponding value taken by at least one of the points in $S$. Similarly we can prove it for the upper bound. Then, we can use the fact that the linear relaxation gives valid bounds for every values within the upper and lower bound to complete the proof. So, we have that 
    
    \begin{equation}
        A^{l+1}_L\cdot x + b^{l+1}_L \leq \mathrm{f}^{l+1}_{NN}(x)_j\leq A^{l+1}_U \cdot x + b^{l+1}_U
    \end{equation}
    
\end{proof}
    
    Then we see that the verifier only certifies the set $S$ to be correctly classified if for all $y \in S$
        $$ (A^U_j \cdot y + b^U_j) \leq (A^L_c \cdot y + b^L_c)$$
    
    Now, we see that from the equation above that if $z \in conv(S)$, then we have that
    $z = \sum_{i=1}^n a_ix_i$, where $x_i \in S$ and $\sum_{i=1}^n a_i = 1, a_i \geq 0$. Then using the above claim we see that
    \begin{equation*}
        \begin{split}
            (\mathrm{f}_{NN}(z))_j &\leq (A^U_j \cdot z + b^U_j) \\ 
                                &= (A^U_j \cdot  \sum_{i=1}^n (a_i x_i) + b^U_j) \\
                                &= \sum_{i=1}^n a_i  (A^U_j \cdot x_i + b^U_j) \\
                                &\leq \sum_{i=1}^n a_i (A^L_c \cdot x_i + b^L_c) \\
                                &= (A^L_c \cdot \sum_{i=1}^n (a_i x_i) + b^L_c) \\
                                &= (A^L_c \cdot z + b^L_c) \\
                                &\leq (\mathrm{f}_{NN}(z))_c \\
        \end{split}
    \end{equation*}
    $$ $$
\end{proof}

\begin{remark}
   For some non-convex attack spaces embedded in high-dimensional pixel spaces, the convex hull of the attack space associated with an image can contain images belonging to a different class (an example of rotation is illustrated in Figure \ref{fig: convex_hull}). Thus, one cannot certify large intervals of perturbations using a single certification cycle of linear relaxation based verifiers.
\end{remark}

\begin{figure}[ht!]
\includegraphics[width=\linewidth]{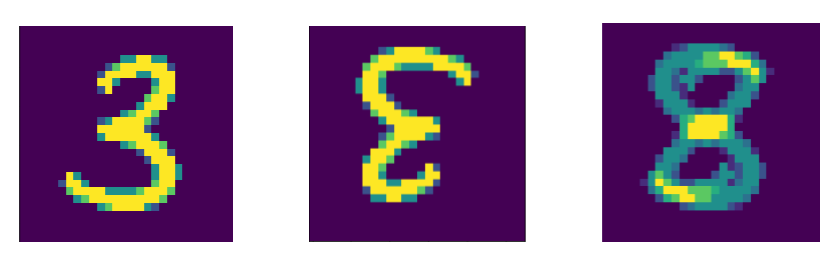}
\caption{Convex Hull in the Pixel Space}
\label{fig: convex_hull}
\end{figure}

\begin{proof} [Proof for Figure \ref{fig: convex_hull}]
    Consider the images given in Figure \ref{fig: convex_hull}, denote them as $x_1, x_2, x_3$ and $x_3 = \frac{x_1 + x_2}{2}$ by construction. We can observe that for an ideal neural network $\mathrm{f}$, we expect that $\mathrm{f}$ classifies $x_1, x_2$ as $3$ and classifies $x_3$ as $8$. Now, we claim that for this network $\mathrm{f}$, it is not possible for a linear-relaxation based verifier to verify that both $x_1, x_2$ are classified as $3$ using just one certification cycle. If it could, then we have by Theorem \ref{thm: Convex hull_2} that we would be able to verify it for the point $x_3 = \frac{x_1 + x_2}{2}$. However, we see that this is not possible as $\mathrm{f}$ classifies $x_3$ as $8$. Therefore,  we need the verification for $x_1$ and for $x_2$ to belong to different certification cycles making input-splitting necessary.
    
\end{proof}

\section{Input Space Splitting}
\label{app:implicit_splitting}

\begin{figure}[ht!]
\centering
\begin{subfigure}[t]{0.4\textwidth}
  \includegraphics[width=\linewidth]{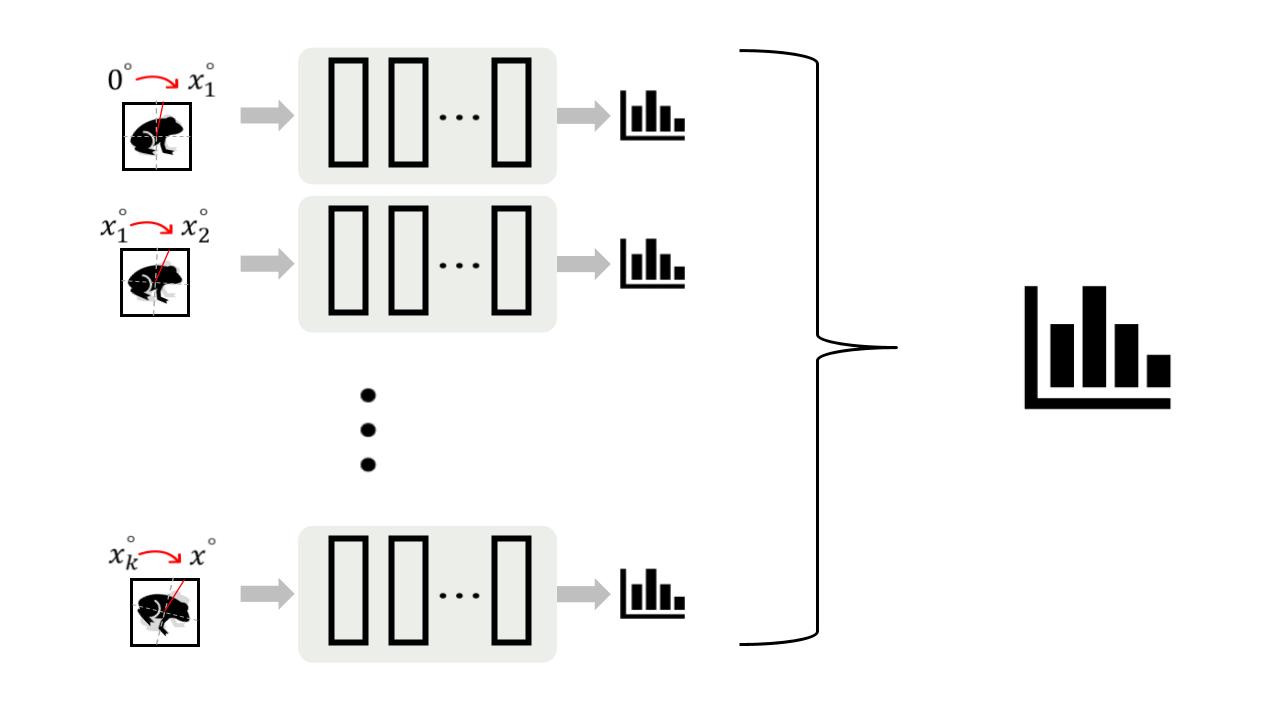}
  \caption{Explicit Splitting}
  \label{fig: explicit_cartoon}
\end{subfigure}
\hfill
\begin{subfigure}[t]{0.4\textwidth}
  \includegraphics[width=\linewidth]{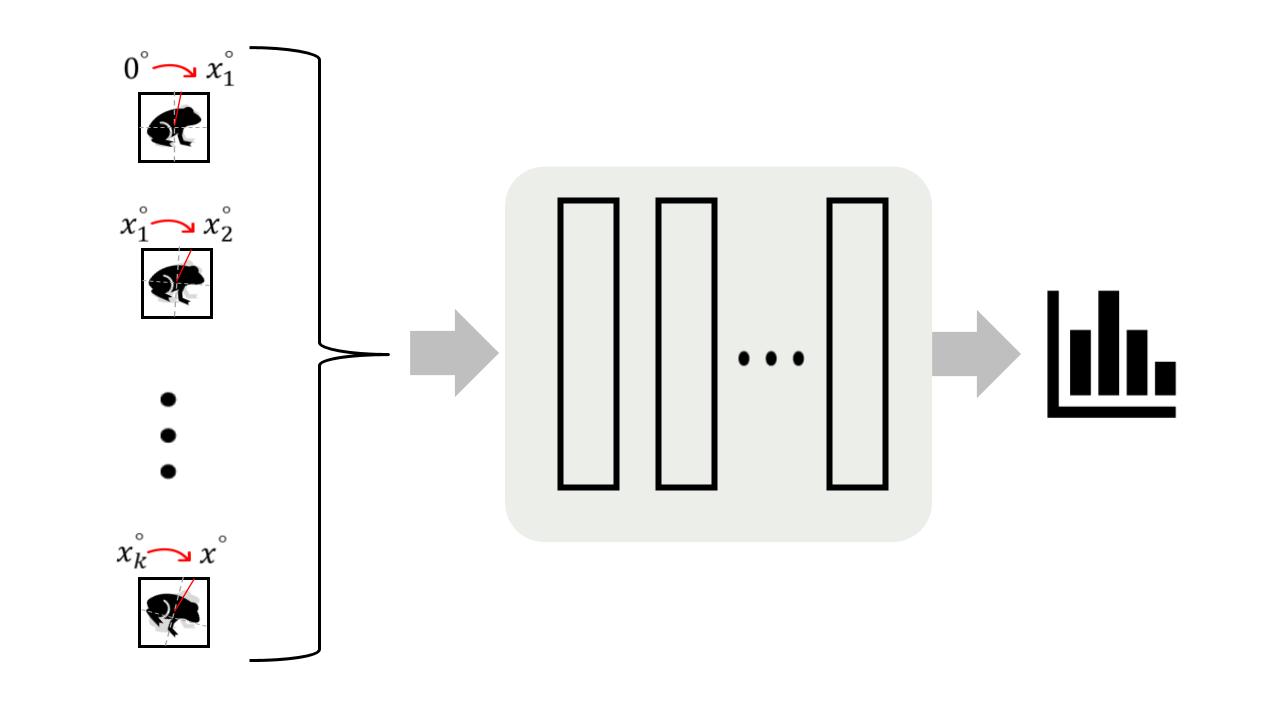}
  \caption{Implicit Splitting}
  \label{fig: implicit_cartoon}
\end{subfigure}
\captionsetup{justification=centering}
\caption{Illustration of refinement techniques.}
\label{fig: split}
\end{figure}

Figure \ref{fig: split} illustrates the difference between explicit and implicit input space splitting.
In Figure \ref{fig:awesome_image1}, we give the form of the activation function for rotation. Even in a small range of rotation angle $\theta$ ($2^\circ$), we see that the function is quite non-linear resulting in very loose linear bounds. Splitting the images explicitly into 5 parts and running them separately (i.e. explicit splitting as shown in Figure \ref{fig: explicit_split}) gives us a much tighter approximation. However, explicit splitting results in a high computation time as the time scales linearly with the number of splits. In order to efficiently approximate this function we can instead make the splits to get explicit bounds on each sub-interval and then run them through certification simultaneously (i.e. implicit splitting as shown in Figure \ref{fig: implicit_split}). As we observe in Figure \ref{fig: implicit_split}, splitting into 20 implicit parts gives a very good approximation with very little overhead (number of certification cycles used are still the same).

Table \ref{tab: implicit_effect} gives a more detailed overview of the effect of implicit splitting. For a large explicit split interval size, we see that using a lot of implicit splits allows us to certify larger radius. However, we also see a pattern that beyond a point adding more implicit splits does not give better bounds. Using implicit splits still results in a single certification cycle. By theorem \ref{thm: Convex hull_2} we see certifying this relaxation is a harder problem than certifying all the rotated images. This could explain the reason we are unable to certify big explicit interval even after using a large number of implicit splits.

\begin{figure}[ht!]
\centering
\begin{subfigure}[]{0.8\linewidth}
  \includegraphics[width=\linewidth]{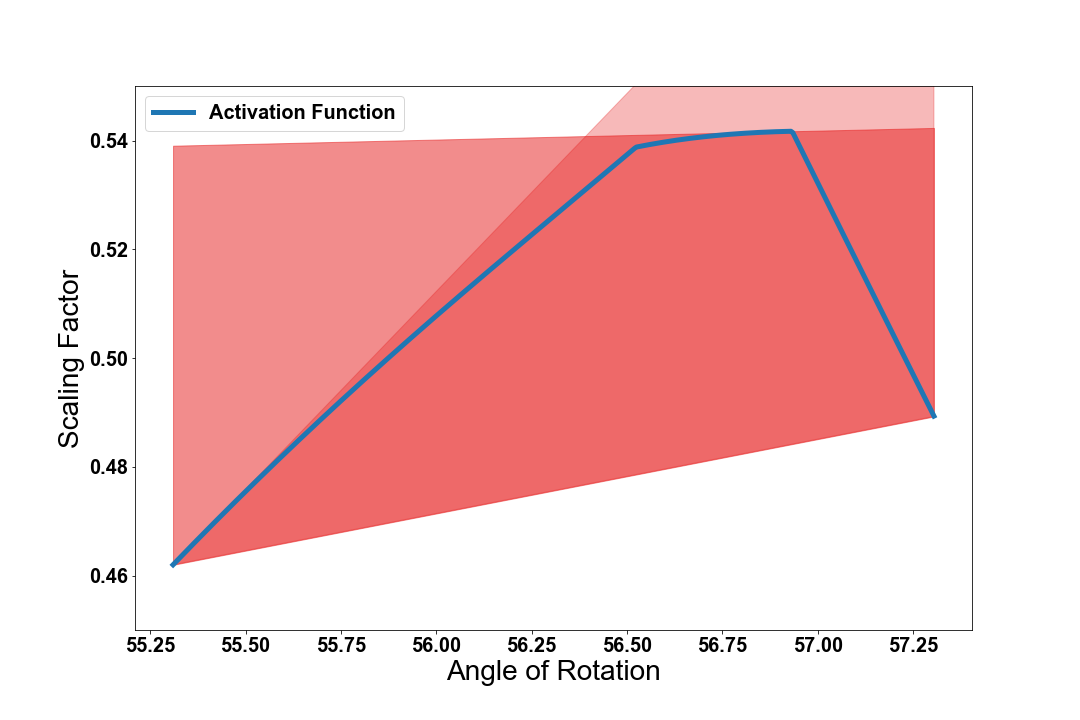}
  \caption{Without splitting the input range}\label{fig:awesome_image1}
\end{subfigure}
\hfill
\begin{subfigure}[]{0.8\linewidth}  \includegraphics[width=\linewidth]{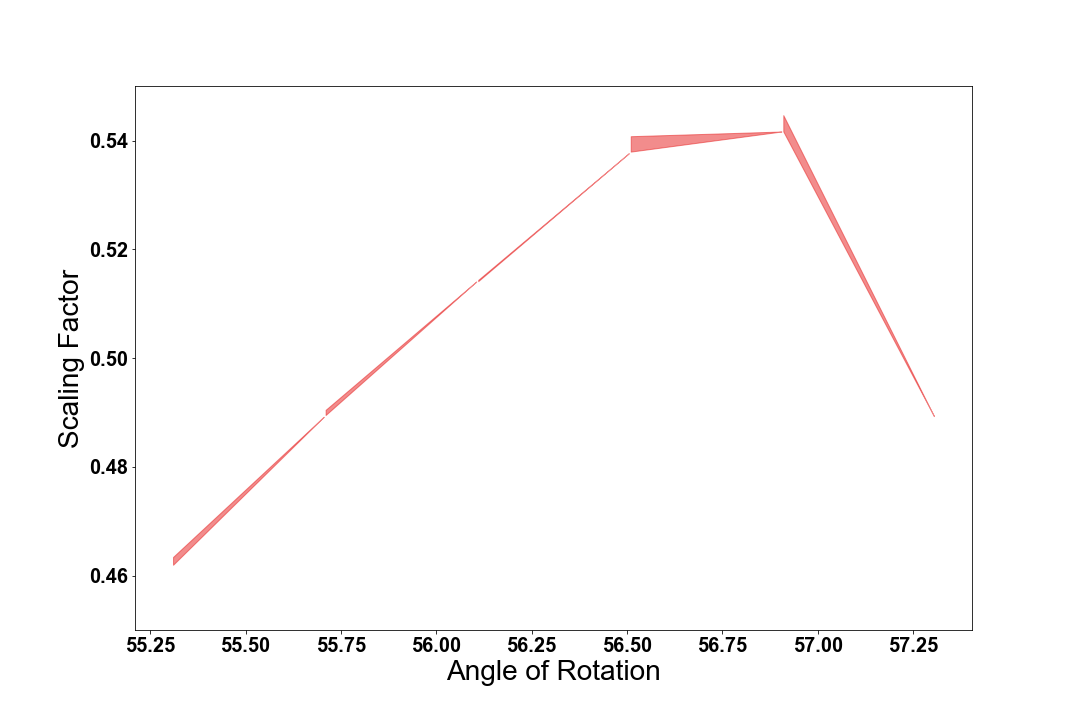}
  \caption{Explicitly splitting the input (5 divisions)}\label{fig: explicit_split}
\end{subfigure}
\hfill
\begin{subfigure}[]{0.8\linewidth} 
  \includegraphics[width=\linewidth]{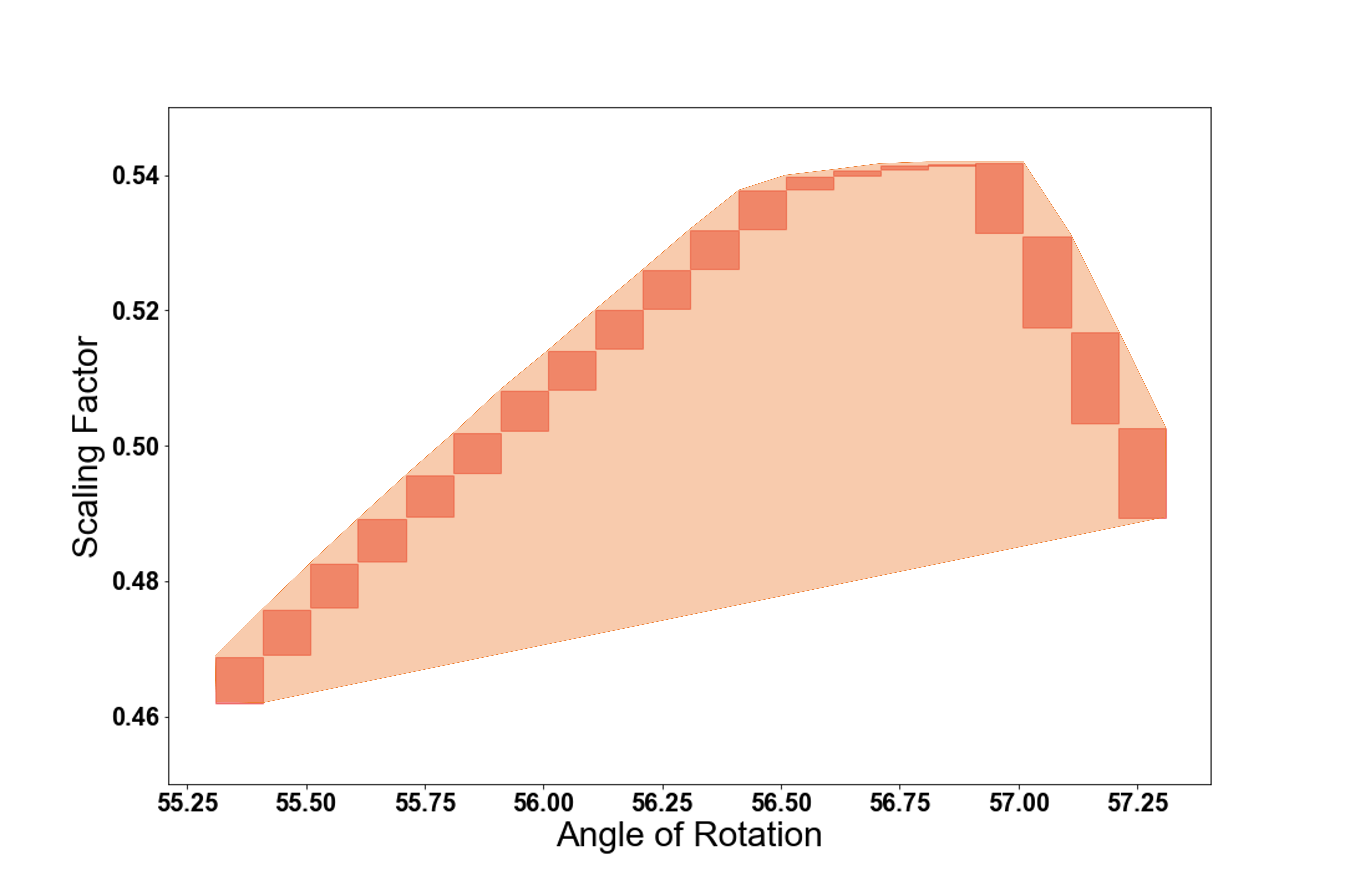}
  \caption{Implicitly splitting the input (20 divisions)}\label{fig: implicit_split}
\end{subfigure}
\captionsetup{justification=centering}
\caption{Bounds for activation function of SP layer in rotation}
\end{figure}

\begin{table}[ht!] 
 \centering
  \captionsetup{justification=centering}
  \caption{Evaluation of averaged certified bounds for rotation space perturbation on MNIST MLP $3 \times 1024$ and 10 images. The results demonstrate the effectiveness of  implicit splits.}
  \begin{adjustbox}{max width=0.8\linewidth}
    \begin{tabular}{|c|cccccc|}
    \hline
    \shortstack{Explicit Split \\ Interval Size} & \multicolumn{6}{c|}{\shortstack{Number of \\Implicit Splits}} \\
    \hline
          &1 &5 &8 &10 &15 &20\\
    \hline 
    \addlinespace[0.1em]
    \multicolumn{7}{|l|}{\bf Experiment (II): Rotations} \\
    \addlinespace[0.1em]
    \hline 
    0.3 &0.27 &50.0 &50.09 &50.12 &50.18 &50.24 \\
    0.5 & 0.0 &40.0 & 50.0 &50.0 & 50.1& 50.2 \\
    0.8 &0.0 &40.0 & 40.0 &40.1 & 50.0& 50.0 \\
    1.0 &0.0 &30.2 & 40.0 &40.0 & 50.0& 50.2 \\
    1.2 &0.0 &10.6 & 40.0 &40.0 & 40.0& 50.0 \\
    1.5 &0.0 &0.0 & 30.15 &40.0 & 40.0& 40.0 \\
    2.0 &0.0 &0.0 &0.4 &30.0 &40.0 &40.0 \\
    3.0 &0.0 &0.0 &0.0 &0.0 &0.9 &30.0 \\
    \hline
    \end{tabular}%
    \end{adjustbox}
  \label{tab: implicit_effect}%
\end{table}


\section{Additional Experimental Results}

\begin{table}[ht!]
  \centering
  \captionsetup{justification=centering}
  \caption{Additional results of Table~\ref{tab:hsl}}
  \begin{adjustbox}{max width=0.9\textwidth}
    \begin{tabular}{|l|cccc||cc||c|}
    \hline
    
    Network & \multicolumn{4}{c||}{Certified Bounds} & \multicolumn{2}{c||}{Ours Improvement (vs Weighted)}  & Attack  \\
    \hline
          &  Naive & Weighted & \bf SPL & \bf SPL + Refine & w/o refine & w/ refine  & Grid  \\
    \hline 
    \addlinespace[0.1em]
    
    \multicolumn{8}{|l|}{\bf Experiment (I)-A: Hue}\\
    \addlinespace[0.1em]
    \hline 
    CIFAR, MLP 5 $\times$ 2048  & 0.00489 &  0.041 & 0.370   & 1.119 & 8.02x  & 26.29x & 1.449   \\
    \hline
    \addlinespace[0.1em]
    \multicolumn{8}{|l|}{\bf Experiment (I)-B: Saturation}\\
    \addlinespace[0.1em]
    \hline 
    CIFAR, MLP 5 $\times$ 2048  & 0.00286 & 0.007 & 0.119   & 0.325 & 16.00x  & 45.42x & 0.346   \\
    \hline
    
    \addlinespace[0.1em]
    \multicolumn{8}{|l|}{\bf Experiment (I)-C: Lightness}\\
    \addlinespace[0.1em]
    \hline 
    CIFAR, MLP 5 $\times$ 2048  & 0.00076 & 0.001 & 0.059   & 0.261  & 58.00x  & 260.00x   & 0.276  \\
    \hline
    \end{tabular}%
    \end{adjustbox}
  \label{tab:hsl app}%
\end{table}

\begin{table}[ht!]
  \centering
  \captionsetup{justification=centering}
  \caption{Additional result of Table~\ref{tab: rotation}}
  \begin{adjustbox}{max width=\textwidth}
    \begin{tabular}{|l|ccc|c||c|}
    \hline
    
    Network & \multicolumn{4}{c||}{Certified Bounds (degrees)}  & Attack (degrees)\\
    \hline
     & \multicolumn{3}{c|}{Number of Implicit Splits}  & \bf SPL + Refine & Grid Attack \\
    \hline
          &  \shortstack{1 implicit \\ No explicit}& \shortstack{5 implicit \\ No explicit} &  \multicolumn{1}{c|}{\shortstack{10 implicit \\ No explicit}}  & \shortstack{100 implicit + \\ \newline explicit intervals of $0.5^\circ$} &\\
    \hline 
    \addlinespace[0.1em]
    \multicolumn{6}{|l|}{\bf Experiment (II): Rotations} \\
    \addlinespace[0.1em]
    \hline 
    MNIST, MLP 4$\times$ 1024   & 0.256 & 0.644 & 1.129 & 46.63  & 48.75 \\
    MNIST, MLP 3$\times$ 1024   & 0.486 & 1.177 & 1.974 & 48.47  & 49.76 \\
    MNIST, CNN 4 $\times$ 5     & 0.437 & 0.952 & 1.447 & 49.20 & 54.61 \\
    \hline
    \end{tabular}%
    \end{adjustbox}
  \label{tab: rotation app}%
\end{table}%

\end{appendices}

\end{document}